\documentclass[11pt]{article}

\usepackage{dsfont}
\def\colorful{0}

\pdfoutput=1

\usepackage{amsthm,amsfonts,amsmath,amssymb,epsfig,color,float,graphicx,verbatim, enumitem}
\usepackage{multirow}
\usepackage[noend]{algpseudocode}
\usepackage{enumitem}
\usepackage[linesnumbered,lined,boxed,commentsnumbered,ruled,vlined]{algorithm2e}

\newif\ifhyper\IfFileExists{hyperref.sty}{\hypertrue}{\hyperfalse}
\hypertrue
\ifhyper\usepackage{hyperref}\fi

\usepackage{enumitem}

\usepackage{framed}
\usepackage{nicefrac}

\def\nnewcolor{1}
\ifnum\nnewcolor=1

\fi
\ifnum\nnewcolor=0

\fi

\ifnum\colorful=1

\else

\fi

\newtheorem{theorem}{Theorem}[section]

\newtheorem{lemma}[theorem]{Lemma}
\newtheorem{informal theorem}[theorem]{Theorem (informal statement)}

\newtheorem{proposition}[theorem]{Proposition}
\newtheorem{corollary}[theorem]{Corollary}
\newtheorem{claim}[theorem]{Claim}
\newtheorem{fact}[theorem]{Fact}

\newtheorem{remark}[theorem]{Remark}

\theoremstyle{definition}
\newtheorem{definition}[theorem]{Definition}
\newcommand{\eqdef}{\stackrel{{\mathrm {\footnotesize def}}}{=}}

\newcommand{\R}{\mathbb{R}}

\newcommand{\Z}{\mathbb{Z}}
\newcommand{\N}{\mathbb{N}}
\newcommand{\E}{\mathbf{E}}

\newcommand{\eps}{\epsilon}

\newcommand{\poly}{\mathrm{poly}}

\newcommand{\sgn}{\mathrm{sign}}

\newcommand{\relu}{\mathrm{ReLU}}

\newcommand\blfootnotea[1]{%
  \begingroup
  \renewcommand\thefootnote{}\footnote{#1}%
  \endgroup
}

\title{Efficiently Learning One-Hidden-Layer ReLU Networks\\ 
via Schur Polynomials\blfootnotea{Author last names are in randomized order.}}

\author{
Ilias Diakonikolas\thanks{Supported by NSF Medium Award CCF-2107079, NSF Award CCF-1652862 (CAREER), 
and a DARPA Learning with Less Labels (LwLL) grant.}\\
University of Wisconsin-Madison\\
{\tt ilias@cs.wisc.edu}\\
\and
Daniel M. Kane\thanks{Supported by NSF Medium Award CCF-2107547, NSF Award CCF-1553288 (CAREER), and a grant from CasperLabs.}\\
University of California, San Diego\\
{\tt dakane@cs.ucsd.edu}
}

\begin{document}

\maketitle

\begin{abstract}
We study the problem of PAC learning a linear combination of $k$ ReLU activations 
under the standard Gaussian distribution on $\R^d$ with respect to the square loss. 
Our main result is an efficient algorithm for this learning task with sample and 
computational complexity $(dk/\eps)^{O(k)}$, where $\eps>0$ is the target accuracy. 
Prior work had given an algorithm for this problem with complexity $(dk/\eps)^{h(k)}$, 
where the function $h(k)$ scales super-polynomially in $k$. Interestingly, 
the complexity of our algorithm is near-optimal within the class of 
Correlational Statistical Query algorithms. 
At a high-level, our algorithm uses tensor decomposition to identify a subspace
such that all the $O(k)$-order moments are small in the orthogonal directions. 
Its analysis makes essential use of the theory of Schur polynomials 
to show that the higher-moment error tensors are small given that the lower-order ones are.
\end{abstract}

\setcounter{page}{0}

\thispagestyle{empty}

\newpage

\section{Introduction} \label{sec:intro}

The efficient learnability of (natural classes of) neural networks has emerged as
one of the central challenges in machine learning.
Despite significant research efforts over several decades ---
see, e.g.,~\cite{Janz15, SedghiJA16, DanielyFS16, ZhangLJ16,
ZhongS0BD17, GeLM18, GeKLW19, BakshiJW19, GoelKKT17, GoelK19, VempalaW19,
DKKZ20,  DK20-ag, ChenKM21, ChenGKM22, CDGKM23} for
some relatively recent works on the topic ---
the classes of neural networks for which provably efficient learning algorithms
are known is startlingly limited. The majority of the aforementioned works
focused on parameter learning --- the task of recovering the weight matrix
of the data-generating neural network --- and consequently require certain
assumptions on the weight matrix (e.g., that it is full-rank with bounded condition number).

Here we focus on the problem of PAC learning,
i.e., approximating the underlying function given access to random labeled examples.
We note that the sample complexity of PAC learning is typically polynomially
bounded for networks of interest without any assumptions on the weight matrix.
The challenging question, of course, is whether a {\em computationally efficient} learner exists.
Arguably, the most basic problem in this setting is that of PAC learning
a single non-linear gate, e.g., a ReLU or sigmoid. This task
has been extensively studied over the past few years
both in the realizable setting (i.e., with consistent labels)
and in the presence of various types of label noise.
A line of research has essentially characterized the complexity of this basic task
under natural assumptions
on the data distribution and the label noise, see, e.g.,~\cite{GoelKKT17, DGKK20, DKZ20, 
DiakonikolasKPZ21, DPT21,DKMR22-neuron, DKRS22, DKTZ22,WZDD23, DKR23}.

In this paper, we study the problem of PAC learning one-hidden-layer ReLU networks
in the realizable setting\footnote{It is easy to see that our results
straightforwardly extend to the case that
the labels have been corrupted by random zero-mean additive noise.}.
A one-hidden-layer ReLU network
is a function $F: \R^d \to \R$ of the form $F(x) = \sum_{i=1}^k w_i \relu(v_i\cdot x)$ for some $w_i \in \R$
and unit vectors $v_i \in \R^d$, where $\relu: \R \to \R$ is defined as $\relu(t) \eqdef \max\{ 0, t\}$.
Following prior work on this problem~\cite{DKKZ20, GoelGJKK20, DK20-ag, ChenKM21,
ChenGKM22, CDGKM23}, we will assume that the feature vectors $x$ are normally distributed.
Despite its apparent simplicity, the complexity of learning this class of functions remains open.

\paragraph{Our Result.}
The main algorithmic contribution of this work is stated in the following theorem.

\begin{theorem}[Main Algorithmic Result]\label{thm:main}
Let $w_i \in \R$, $i \in [k]$, with $\sum_{i=1}^k |w_i| \leq 1$
and $v_i \in \R^d$, $i \in [k]$, be unit vectors.
Define a function $F:\R^d\rightarrow \R$ by
$F(x) = \sum_{i=1}^k w_i \relu(v_i\cdot x)$.
Let $X \sim N(0, I)$.
Then for $C$ a sufficiently large universal constant,
there exists an algorithm that given $\eps>0$ sufficiently small
and $N=(dk/\eps)^{Ck}$ i.i.d.\ samples of the form $(X,F(X))$,
runs in $\poly(N)$ time and outputs a function $\tilde F : \R^d\rightarrow \R$
such that with probability $9/10$ we have
$\|\tilde F(X) - F(X)\|_2 \leq \eps$.
\end{theorem}

A few remarks are in order regarding Theorem \ref{thm:main}. 
First, we note that the assumption that the sum of the absolute values of the weights $w_i$
be bounded is somewhat strong, but turns out to be necessary. 
One might instead hope that for any set of weights one could learn a function $\tilde F$ 
such that $\|\tilde F(X) - F(X)\|_2 \leq \eps \|F(X)\|_2$. Unfortunately, this is information-theoretically 
impossible. Consider for example the function 
$F(x) = \relu(v\cdot x) + \relu(u\cdot x) - \relu((v+u)\cdot x)$, 
where this last term is actually given as $\|v+u\|_2\relu((v+u)/\|v+u\|_2 \cdot x)$. 
In such a case, $F(x)$ would be $0$ unless $\sgn(v\cdot x) \neq \sgn(u\cdot x)$. 
If $v$ and $u$ are close to each other, this event could happen with arbitrarily small probability. 
Thus, to learn $F$ to such a relative error guarantee would require an unbounded number of samples.

Second, it is worth mentioning that our learning algorithm is not proper, i.e., the hypothesis, $\tilde F$, 
returned is not a one-hidden-layer ReLU network. 
The hypothesis $\tilde F$ is a somewhat more complicated function that can still be evaluated 
at any point of interest in time $(dk/\eps)^{O(k)}$. While we do not prove any relevant theorem here, 
we believe that with some additional work (involving runtime $(dk/\eps)^{O(k^2)}$) 
one can adapt our algorithm to output  a nearly proper hypothesis, which is a sum of 
slightly smoothed versions of ReLUs.

Finally, we note that our result is not specific to ReLUs. 
Similar techniques should apply to any function of the form 
$F(x) = \sum_{i=1}^k w_i \sigma(v_i \cdot x)$ for $w_i$ real numbers 
with $\sum_{i=1}^k |w_i|$ not too large, $v_i$ unit vectors, and $\sigma$ a known activation 
function satisfying mild conditions on its Fourier spectrum.

\paragraph{Comparison to Prior Work.} 
Before we describe our algorithmic approach,
we provide a brief summary and comparison with the most relevant prior work.
The first positive result on PAC learning one-hidden-layer ReLU networks was obtained in~\cite{DKKZ20}.
That work gave a PAC learning algorithm with complexity $\poly(d/\eps)+(k/\eps)^{O(k^2)}$
for the special case that the weights $w_i$ are positive. Subsequently,~\cite{DK20-ag} gave
a significantly improved algorithm
for the positive weights case with complexity $\poly(d/\eps)+(k/\eps)^{O(\log^2(k))}$.
\cite{ChenKM21} gave a fixed-parameter tractable algorithm
for learning ReLU networks of constant depth,
albeit with complexity exponential in $1/\eps$.

The most directly related prior work is that of~\cite{CDGKM23} who gave an algorithm
for one-hidden-layer networks (in the exact same setting as Theorem~\ref{thm:main})
with sample and computational complexity $(d/\eps)^{h(k)}$,
where $h(k)= k^{O(\log^2(k))}$. In comparison, our algorithm of
Theorem~\ref{thm:main} improves the super-polynomial dependence on $k$ in the exponent
to linear.

It is worth noting that the complexity of our algorithm is essentially optimal within the class
of Correlational Statistical Query (CSQ) algorithms.
CSQ algorithms are a subclass of SQ algorithms~\cite{Kearns:98}
capturing many learning algorithms used in practice --- including, e.g., gradient descent on the square loss.
A CSQ algorithm is allowed to choose any bounded query function on the
examples and obtain estimates of its correlation with the labels.
Interestingly, \cite{DKKZ20} (see also \cite{GoelGJKK20} for a weaker bound)
showed that any CSQ algorithm for our learning task requires complexity $d^{\Omega(k)}$,
nearly matching our upper bound. It can be readily verified that both our algorithm and the algorithms
in the prior works~\cite{DKKZ20, DK20-ag, CDGKM23} are CSQ algorithms.

\paragraph{Our Techniques.}
At a very high level, our techniques bear similarities to
a number of prior works in this area~\cite{DKKZ20, DK20-ag, CDGKM23}.
Let $F(x)= \sum_{i=1}^k w_i \relu(v_i\cdot x)$ be the target (label generating) function.
We note that if $V$ is the vector space spanned by the $v_i$'s,
then $F(x)$ depends only on $\mathrm{Proj}_V(x)$.
This means that if we could learn the $k$-dimensional subspace $V$,
we can use brute-force --- in this case, approximating $F$ as a low-degree polynomial
in $\mathrm{Proj}_V(x)$ using $L_2$ regression --- to learn $F$ efficiently.

To learn $V$, we use the method of moments.
The $t$-th moment tensor of $F$, properly conditioned, is $0$ if $t>1$ is odd
and proportional to $\sum_{i=1}^k w_i v_i^{\otimes t}$ if $t$ is even
(see Corollary~\ref{cor:relu-corr} and Equation~\eqref{eqn:moment-tensor}).
In particular,
this quantity lies in $V^{\otimes t}$
and we would like to use this fact to find $V$.
To achieve this, we can think of this moment tensor as a matrix that takes
a $(t-1)$-order tensor and returns a vector.
Then $V$ should contain the span of this matrix, which we can efficiently compute.
By taking the sum of these spans for various values of $t$, we can hope to learn $V$.
We note that it is necessary to consider moment tensors of order $t$ up to at least $C k$,
where $C>0$ is a sufficiently large universal constant. 
Otherwise, the CSQ lower bound construction
of~\cite{DKKZ20} implies that this approach will necessarily fail. 
In particular, computing these $\Omega(k)$ moments will require $d^{\Omega(k)}$ time 
even to write down the answer, and this is a major contributing factor in our final runtime.

Unfortunately, the above approach would only work 
if we could approximate the $t$-th moment tensors
of $F$ {\em exactly}. Of course, all we can hope for is to learn them approximately.
However, fortunately, most of the aforementioned plan should still work
if instead of considering the span of the higher order moment tensors,
we look at the top few right singular vectors.
However, this brings us to another problem. The top few singular vectors
will only robustly produce an approximation of $V$ if
the corresponding singular values were not too small
(or were not smaller than the error in our approximation of the moment tensors).
This could become problematic if, for example,
all of the $v_i$'s nearly lie in a proper subspace of $V$,
or if there are two or more vectors whose terms nearly cancel out.
If such situations occur, it means that
{\em even information-theoretically} we cannot hope to recover
a reasonable approximation of $V$.

What we can hope to accomplish instead
is to learn a subspace $W$ such that the ``low-order" moment tensors
are small in all directions orthogonal to $W$.
Fortunately, this turns out to be sufficient for our purposes.
Once we have learned $W$, we just need to show that $F(x)$ is well-approximated
by some function of $\mathrm{Proj}_W(x)$.
Investigating this in terms of moments boils down to showing that
$\mathrm{Err}_t := \| \sum_{i=1}^k w_i (v_i^{\otimes t} - \mathrm{Proj}_W(v_i)^{\otimes t}) \|_2$
is small for all (not too large) even values $t$. Fortunately, by the way
we computed $W$ (considering the first $O(k)$ many moments), it follows that
$\mathrm{Err}_t$ is quite small for $t = O(k)$. Perhaps surprisingly,
it turns out (see Proposition~\ref{prop:p-m}) that this actually suffices
to show that $\mathrm{Err}_t$ is also small for all (not too large) values of $t$.
In particular, we use the theory of Schur polynomials
(Definition~\ref{def:sp} and Corollary~\ref{main bound cor})
to re-express the $t$-th order tensor in question here as a sum
of not-too-many tensor powers of $v_i$'s and $\mathrm{Proj}_W(v_i)$'s times
the low-order versions of this tensor whose norms are small by construction.

\section{Preliminaries} \label{sec:prelims}

\paragraph{Notation.}
For $n \in \Z_+$, we denote by $[n]$ the set $\{1, 2, \ldots, n\}$.
For a vector $v \in \R^n$, let $\| v \|_2$ denote its Euclidean norm.
We denote by $x \cdot y$ the standard inner product between $x, y \in \R^d$.
We will denote by $\delta_0$ the Dirac delta function
and by $\delta_{i, j}$ the Kronecker delta.
Throughout the paper, we let $\otimes$ denote the tensor/Kronecker product.
For a vector $x \in \R^d$, we denote by $x^{\otimes m}$ the $m$-th order tensor power of $x$.

We will denote by $N(0, I_d)$ the $d$-dimensional Gaussian distribution
with zero mean and identity covariance; we will use $N(0, I)$ when the underlying dimension
will be clear from the context. We will use $N(0, 1)$ for the univariate case.
For a random variable $X$ and $p \geq 1$,
we will use $\|X\|_p \eqdef \E[|X|^p]^{1/p}$ to denote its {\em $L_p$-norm}.


Let $V$ be an inner product space. If $A$ and $B$ are elements of $V^{\otimes t}$
for some $t \in \Z_+$, then we use  $\langle A, B \rangle$ to denote the inner product of $A$
and $B$ induced by the inner product on $V$. We also use $\|A\|_2 = \langle A, A \rangle^{1/2}$
for the corresponding $\ell_2$-norm.

\paragraph{Hermite Analysis and Concentration.} \label{ssec:hermite}
Consider $L_2(\R^d, N(0, I))$, the vector space of all
functions $f : \R^d \to \R$ such that $\E_{x \sim N(0, I)}[f(x)^2] <\infty$.
This is an inner product space under the inner product
$\langle f, g \rangle = \E_{x \sim N(0, I)} [f(x)g(x)] $.
This inner product space has a complete orthogonal basis given by
the \emph{Hermite polynomials}.
In the univariate case, we will work with normalized Hermite polynomials
defined below.

\begin{definition}[Normalized Probabilist's Hermite Polynomial]\label{def:Hermite-poly}
For $k\in\N$, the $k$-th \emph{probabilist's} Hermite polynomial
$He_k:\R\to \R$ is defined as
$He_k(t)=(-1)^k e^{t^2/2}\cdot\frac{\mathrm{d}^k}{\mathrm{d}t^k}e^{-t^2/2}$.
We define the $k$-th \emph{normalized} probabilist's Hermite polynomial
$h_k:\R\to \R$ as
$h_k(t)=He_k(t)/\sqrt{k!}$.
\end{definition}

\noindent Note that for $G\sim N(0,1)$ we have $\E[h_n(G)h_m(G)] = \delta_{n,m}$,
and $ \sqrt{m+1} h_{m+1}(t) = t h_m(t) - h'_m(t)$.

We will use multivariate Hermite polynomials in the form of
Hermite tensors. We define the normalized Hermite tensor as follows,
in terms of Einstein summation notation.

\begin{definition}[Normalized Hermite Tensor]\label{def:Hermite-tensor}
For $k\in \N$ and $x \in V$ for some inner produce space $V$, 
we define the $k$-th Hermite tensor as
\[
(H_k^{(V)}(x))_{i_1,i_2,\ldots,i_k}:=\frac{1}{\sqrt{k!}}\sum_{\substack{\text{Partitions $P$ of $[k]$}\\ \text{into sets of size $1$ and $2$}}}\bigotimes_{\{a,b\}\in P}(-I_{i_a,i_b})\bigotimes_{\{c\}\in P}x_{i_c} \;,
\]
where $I$ above denotes the identity matrix over $V$. Furthermore, if $V=\R^d$, 
we will often omit the superscript and simply write $H_k(x)$.
\end{definition}

\noindent 
We will require a few properties that follow from this definition. 
First, note that if $V$ is a subspace of $W$, 
then $H^{(V)}_k(\mathrm{Proj}_V(x)) = \mathrm{Proj}_V^{\otimes k}H_k^{(W)}(x)$. 
Applying this when $V$ is the one-dimensional subspace spanned by a unit vector $v$ 
gives that 
$\langle H_k(x), v^{\otimes k} \rangle = h_k(v \cdot x)$.
We will also need to know that the entries of $H_k(x)$ 
form a useful Fourier basis of $L^2(\R^d,N(0,I))$.  
In particular, for non-negative integers $m$ and $k$, we have that
$\E_{x \sim N(0, I)}[H_k(x) \otimes H_{m}(x)]$ is $0$ 
if $m\neq k$ and  $\mathrm{Sym}_k(I_{d^k})$, if $m=k$,
where $\mathrm{Sym}_k$ is the symmetrization operation over the first $k$ coordinates. 
From this we conclude that if $T$ is a symmetric $k$-tensor, 
then $\E_{x \sim N(0, I)}[\langle H_k(x), T \rangle H_m(x)]$ 
is $0$ if $m\neq k$ and $T$ if $m=k$.

For a polynomial $p: \R^d \to \R$, we will use $\|p\|_r \eqdef \E_{x \sim N(0, I)} [|p(x)|^r]^{1/r}$,
for $r \geq 1$. We recall the following well-known hypercontractive
inequality~\cite{Bon70,Gross:75}:
\begin{fact} \label{thm:hc}
Let $p: \R^d \to \R$ be a degree-$k$ polynomial and $q>2$.  Then
$\|p\|_q \leq (q-1)^{k/2} \|p\|_2$.
\end{fact}

\section{Technical Results} \label{sec:technical-results}

In this section, we establish some structural results that are used
in our algorithm and its analysis. 

\subsection{Hermite Analysis of ReLUs and Moment-Tensor Estimation} \label{ssec:hermite}

\begin{lemma} \label{lem:univ-relu-corr}
For $G \sim N(0, 1)$ and $m \in \Z_+$, we have that
$\E[\relu(G) h_m(G)] = c_m$ for some $c_m \in \R$.
Specifically, if $m>1$, then $c_m=0$ if $m$ is odd
and $$c_m = (-1/4)^{(m-2)/4}\sqrt{\binom{m-2}{(m-2)/2}}/\sqrt{2\pi m(m-1)} = \Theta(m^{-5/4})$$
if $m$ is even.
\end{lemma}
\begin{proof}
Let $g(t)$ be the probability density function (pdf) of $N(0, 1)$.
We need to evaluate the quantity
\begin{equation}\label{ReLU integral equation}
\int_{-\infty}^\infty \relu(t) h_m(t) g(t)dt \;.
\end{equation}
Note that
$$
\frac{\mathrm{d}}{\mathrm{d} t}\left(h_m(t) g(t) \right) = (h_m'(t) g(t) - t h_m(t) g(t))
= -\sqrt{m+1}h_{m+1}(t) g(t) \;,
$$
where we used the recurrence relation $\sqrt{m+1} h_{m+1}(t) = t h_m(t) - h'_m(t)$.
Thus, using integration by parts and noting that the limits at infinity are asymptotically zero,
we find that \eqref{ReLU integral equation} equals:
$$
\int_{-\infty}^\infty \relu'(t) h_{m-1}(t)g(t)/\sqrt{m} dt \;.
$$
Integrating by parts again yields
$$
\int_{\infty}^\infty \relu''(t) h_{m-2}(t) g(t)/\sqrt{m(m-1)} dt \;.
$$
Note that $\relu''(t) = \delta_0(t)$. Thus, this integral is equal to
$$
h_{m-2}(0)g(0)/\sqrt{m(m-1)} \;.
$$
For odd $m$, we have that $h_{m-2}(0)=0$, which implies that $c_m = 0$.
For even $m$, we have that 
$c_m = (-1/4)^{(m-2)/4}\sqrt{\binom{m-2}{(m-2)/2}}/\sqrt{2\pi m(m-1)}$, as was to be shown.
This completes the proof of Lemma~\ref{lem:univ-relu-corr}.
\end{proof}

We also require the following
high-dimensional analogue of Lemma~\ref{lem:univ-relu-corr}.

\begin{lemma} \label{lem:relu-exp}
For any unit vector $v \in \R^d$ and $x \in \R^d$, we have that
$
\relu(v \cdot x) = \sum_{m=0}^\infty c_m \langle H_m(x), v^{\otimes m} \rangle \;.
$
\end{lemma}
\begin{proof}
By Lemma \ref{lem:univ-relu-corr} we have that
$
\relu(v\cdot x) = \sum_{m=0}^\infty c_m h_m(v\cdot x) = 
\sum_{m=0}^\infty c_m \langle v^{\otimes m},H_m(x)\rangle \;.
$
\end{proof}

\noindent Via orthogonality, as an immediate corollary we obtain:

\begin{corollary} \label{cor:relu-corr}
For a unit vector $v \in \R^d$, $X \sim N(0, I_d)$, and $m \in \Z_+$
we have
$$ \E[\relu(v\cdot X) H_m(X)] = c_m v^{\otimes m} \;.$$
\end{corollary}

We will also need a way of algorithmically approximating the 
Hermite components of a function.
\begin{lemma}\label{moment computation lemma}
Let $X \sim N(0, I_d)$ and let $y$ be a (possibly correlated) real-valued random variable.
Let $m \in \Z_+$, $\delta \in (0, 1)$, and $t>2$.
There exists an algorithm that given $N = O\left(\binom{d+m}{m}e^{O(m/t)}\|y\|_t^2/(\tau^2\delta^2) \right)$
independent samples from $(X,y)$,
runs in sample polynomial time, and computes an estimate of $\E[y H_m(X)]$
whose $\ell_2$-error at most $\delta$
with probability at least $1-\tau$.
\end{lemma}
\begin{proof}
The algorithm is simply to use the empirical estimator.
In order to get the appropriate $\ell_2$-error,
we need that the sum of the squared errors of the empirical estimates of $y h_{\alpha}(X)$
is at most $\delta^2 \tau^2$, where $h_{\alpha}(X) = \prod_{i=1}^d h_{\alpha_i}(X_i)$
for $\alpha \in \N^d$ with $\sum_{i=1}^d \alpha_i = m$.
To do this, we note that the expected sum of squared empirical errors is at most
$\sum_{\alpha} \|y h_{\alpha}(X) \|_2^2/N$, and that as long as this is at most
$\delta^2 \tau^2$,
our desired statement follows by Markov's inequality.

It remains to show that $\sum_{\alpha} \|y h_{\alpha}(X) \|_2^2/N \leq \delta^2 \tau^2$
for appropriately large $N$. To prove this, we note that there are fewer than $\binom{m+d}{m}$
many possible values of $\alpha$, and each of the terms has size at most
$$
\|y h_\alpha(X) \|_2^2 \leq \|y\|_t^2 \|h_\alpha(X)\|_{1/(1/2-1/t)}^2
$$
by H{\"o}lder's inequality.
Noting that $1/(1/2-1/t) = 2 + O(1/t)$, by hypercontractivity (Fact~\ref{thm:hc})
we have that
$\|h_\alpha(X)\|_{1/(1/2-1/t)} = (1+O(1/t))^{m/2} \|h_\alpha(X)\|_2 = e^{O(m/t)},$
and the lemma follows.
\end{proof}

\subsection{Schur Polynomials and Key Technical Result} \label{sec:schur}

The analysis of our algorithm will make essential use of Schur polynomials
and their properties. We start by recalling the definition of Schur polynomials.

\begin{definition}[Schur Polynomials] \label{def:sp}
Let $\lambda_1 \geq \lambda_2 \geq \cdots \geq \lambda_n \geq 0$
be a sequence of non-negative integers denoted by $\lambda$.
The Schur polynomial $s_\lambda(x)$ is a polynomial in $n$ variables $x = (x_1, \ldots, x_n)$
given by
\begin{equation} \label{eqn:schur}
s_\lambda(x_1,\ldots,x_n) \eqdef
\frac{\det\left(\left[x_i^{\lambda_j+j-1} \right]_{1\leq i,j\leq n} \right)}{\det\left(\left[x_i^{j-1} \right]_{1\leq i,j\leq n} \right)} \;.
\end{equation}
\end{definition}

The first Jacobi-Trudi formula, stated below, expresses the Schur polynomials
as a determinant in terms of the complete homogeneous symmetric polynomials.

\begin{fact}[First Jacobi-Trudi Formula] \label{fact:jt}
We have that
$$
s_\lambda(x) = \det([y_{\lambda_i+j-i}(x)]_{1\leq i,j \leq n}) \;,
$$
where $y_k(x)$ is the complete homogeneous symmetric polynomial of degree $k$
given as the sum of all of the degree-$k$ monomials in $(x_1,\ldots,x_n)$.
\end{fact}

\begin{remark}
{\em The complete  homogeneous symmetric polynomial of degree $k$ is usually denoted $h_k$,
which we have avoided in order to not cause confusion with our notation for Hermite polynomials.}
\end{remark}

\noindent Lastly, we will also need that these are polynomials 
with non-negative coefficients:
\begin{fact}\label{positive coef fact}
The Schur polynomial $s_\lambda(x)$ is a polynomial in $x$, 
homogeneous of degree $|\lambda| = \sum_i \lambda_i$, with non-negative coefficients.
\end{fact}

\noindent Making use of the theory of Schur polynomials will be essential 
in proving that our higher moment error tensors are not too large 
given that the lower order ones are not. In particular, we prove a general result 
about certain exponential sequences of tensors. 
As a warmup, we begin with a scalar version of the statement we require. 

\begin{proposition}\label{higher moments in terms of lower prop}
For $k \in \Z_+$, let $w_i \in \R$ and $x_i \in \R$, $i \in [k]$,
with $|x_i| \leq 1$.
For $t \in \N$, let $M_t \eqdef \sum_{i=1}^k w_i x_i^t$.
Then, for $t\geq k$, we have that
$$
|M_t| \leq  \binom{t}{k-1}(2k)^k\max_{t<k}(|M_t|) \;.
$$
\end{proposition}
\begin{proof}
We begin by proving the desired statement
in the special case where no two of the $x_i$'s are identical.
As any collection of $x$'s can be written as a limit of such situations,
this will suffice by continuity.

Let $w = (w_1, \ldots, w_k)$ and
let $X_t = (x_1^t, x_2^t,\ldots,x_{k}^t)$
so that $M_t = w\cdot X_t$. Since $X_0,\ldots,X_{k-1}$ are linearly independent
(by the non-vanishing property of the Vandermonde determinant),
it follows that any $X_t$ can be written as a linear combination of $X_0,\ldots,X_{k-1}$.
The following claim establishes bounds on the coefficients
of the corresponding linear combination.
\begin{claim} \label{clm:coeff}
For any $t \in \N$, we have that
$X_t = \sum_{a=0}^{k-1} c_a X_a$, where
$c_a = (-1)^{k+a+1} s_\lambda(x_1,\ldots,x_k)$ and
$\lambda = (\lambda_1, \ldots, \lambda_k)$
with $\lambda_1 = (t-k+1)$, $\lambda_j = 1$ for $2 \leq j \leq k-a$,
and $\lambda_j = 0$ otherwise.
Moreover, we have that the sum of the absolute values of the coefficients 
of these $s_\lambda$ is at most $\binom{t}{k-1}(2k)^k$.
\end{claim}
\begin{proof}[Proof of Claim~\ref{clm:coeff}]
By Cramer's rule, 
the coefficient of $X_a$, $ a \in \{0, \ldots, k-1\}$, in this linear combination
will be
\begin{eqnarray*}
&&c_a = \det([X_{k-1},X_{k-2},\ldots,X_{a+1},X_t,X_{a-1},\ldots,X_0])/\det([X_{k-1},\ldots,X_0]) = \\
&=& (-1)^{k+a+1}\det([X_t,X_{k-1},X_{k-2},\ldots,X_{a+1},X_{a-1},\ldots,X_0])/\det([X_{k-1},\ldots,X_0]) \\
&=& (-1)^{k+a+1}s_\lambda(x_1,\ldots,x_k) \;,
\end{eqnarray*}
where $\lambda = (\lambda_1, \ldots, \lambda_k)$ is the partition
with first coordinate equal to $(t-k+1)$,
followed by $(k-1-a)$ many $1$'s, followed by a sequence of $0$'s.

By Fact \ref{positive coef fact}, we have that $s_\lambda$ has non-negative coefficients, 
and thus the sum of the absolute values of these coefficients is merely 
$s_\lambda(\mathbf{1})$, where $\mathbf{1}=(1,1,\ldots,1).$

We can bound $s_\lambda(\mathbf{1})$ using the 
first Jacobi-Trudi formula
(Fact~\ref{fact:jt}) 
and note that $|y_m(\mathbf{1})| = \binom{k+m-1}{k-1}$. 
Since the sequence $\binom{k+m}{m}$ is log-concave in $m$,
the traversal of the matrix $[y_{\lambda_i+j-i}(x)]_{1\leq i,j \leq n}$
with the largest product of terms gives an absolute value of at most $\binom{t}{k-1}k^{k-1-a}$.
Furthermore, since each term with $i>j+1$ has $y_{\lambda_i+j-i}=0$
(since the subscript will be negative),
there are at most $2^{k-1-a}$ many non-vanishing traversals in the expansion of the determinant.
Thus, the sum of the absolute values of the coefficients 
of the $s_\lambda$ in $c_a$ is at most $\binom{t}{k-1}(2k)^{k-1-a}$.
Summing over $a$ proves the claim.
\end{proof}

Note that
$$
M_t = \sum_{a=0}^{k-1} c_a w \cdot X_a = \sum_{a=0}^{k-1} c_a M_a \;,
$$
which has absolute value at most $\binom{t}{k-1}(2k)^k\max_{t<k}(|M_t|)$ 
by Claim~\ref{clm:coeff}.
This completes the proof of Proposition~\ref{higher moments in terms of lower prop}.
\end{proof}

We will actually require a somewhat stronger tensor-valued version
of Proposition \ref{higher moments in terms of lower prop}.

\begin{proposition}\label{general bound prop}
Let $V$ be an inner product space with norm $\| \cdot \|_2$.
Let $w_i \in \R$, and $v_i  \in V$, $i \in [k]$,
with $\| v_i \|_2 \leq 1$ for all $i \in [k]$.
For $t \in \N$, let $M_t \in V^{\otimes t}$ be the tensor
$\sum_{i=1}^k w_i v_i^{\otimes t}$. Then, for $t\geq k$, we have that
$$
\|M_t\|_2 \leq \binom{t}{k-1}(2k)^k\max_{t<k}(\|M_t\|_2) \;.
$$
\end{proposition}

\begin{proof}
Our goal is to write $M_t$ as a combination of $M_0,\ldots,M_{k-1}$,
as in the proof of Proposition~\ref{higher moments in terms of lower prop}.
To accomplish this, we need to define tensor-valued Schur polynomials.
In particular, if $\lambda$ is a partition with at most $k$ parts,
we define $s_{\lambda}(v_1,v_2,\ldots,v_k)$ as the order-$|\lambda|$ tensor
(where $|\lambda| = \sum_i \lambda_i$) 
obtained by replacing each monomial
$c_\alpha \prod_{i=1}^k x_i^{\alpha_i}$ in the usual Schur polynomial
with the symmetrization of the tensor $c_\alpha \bigotimes_{i=1}^k v_i^{\otimes \alpha_i}$.

We claim that for $t\geq k$ we have
\begin{equation}\label{tensor recursion eqn}
M_t = \mathrm{Sym}\left(\sum_{a=0}^{k-1} (-1)^{k+a+1} M_a \otimes s_{(t-k-1,\overbrace{1,1,\ldots,1}^{k-1-a})}(v_1,\ldots,v_k) \right) \;,
\end{equation}
where $\mathrm{Sym}$ denotes the symmetrization operator that averages 
a tensor over all permutations of its entries.
To show this, we note that both sides of Equation~\eqref{tensor recursion eqn}
are symmetric order-$t$ tensors. Furthermore, for any vector $u$,
the inner product of the left hand side with $u^{\otimes t}$ is
$$
\sum_{i=1}^k w_i (v_i\cdot u)^t \;,
$$
while the inner product with the right hand side is
$$
\sum_{a=0}^{k-1} (-1)^{k+a+1} \langle M_a, u^{\otimes t} \rangle s_{(t-k-1,\overbrace{1,1,\ldots,1}^{k-1-a})}(v_1\cdot u,\ldots,v_k\cdot u) \;.
$$
By applying Claim~\ref{clm:coeff}
to
$$
N_t \eqdef \sum_{i=1}^k w_i (v_i\cdot u)^t
$$
implies that these quantities are equal. Consequently, the difference between
the left and right hand sides of \eqref{tensor recursion eqn} is a symmetric tensor
that is orthogonal to all $u^{\otimes t}$, and therefore must be identically $0$.

From here the proof follows fairly easily.
Indeed, Equation \eqref{tensor recursion eqn} implies that
\begin{align*}
\|M_t\|_2 & \leq k \max_{0\leq a \leq k-1} \|M_a\|_2 
\sum_{a=0}^{k-1} \left\| s_{(t-k-1,\overbrace{1,1,\ldots,1}^{k-1-a})}(v_1,\ldots,v_k) \right\|_2\\
& \leq \binom{t}{k-1}(2k)^k \max_{t<k}(\|M_t\|_2) \;,
\end{align*}
where the second line follows from the fact that, by Claim \ref{clm:coeff}, 
the sum of the absolute values of the coefficients of the relevant $s_\lambda$'s 
is bounded and that each of these monomials produces a tensor of norm at most $1$.
This completes the proof of Proposition~\ref{general bound prop}.
\end{proof}

We will need the following slight generalization of the above:

\begin{corollary}\label{main bound cor}
Let $V$ be an inner product space with norm $\| \cdot \|_2$.
Let $w_i \in \R$, $i \in [k]$, and $v_i  \in V$
with $\|v_i\|_2\leq 1$ for all $i \in [k]$.
For $t \in \N$, let $M_t \in V^{\otimes t}$ be the tensor $\sum_{i=1}^k w_i v_i^{\otimes t}$.
Then, for even $t\geq 2k$, we have that
$$
\|M_t\|_2 \leq \binom{t}{k-1}(2k)^k\max_{t<2k, \;  \mathrm{ even}}(\|M_t\|_2) \;.
$$
\end{corollary}
\begin{proof}
This follows by noting that
$$
M_{2t} = \sum_{i=1}^k w_i (v_i^{\otimes 2})^{\otimes t}
$$
and applying Proposition \ref{general bound prop} to $M_{2t}$
thought of as an element of $(V^{\otimes 2})^{\otimes t}$
given by a linear combination of the $t^{th}$ tensor powers of $v_i^{\otimes 2}$.
\end{proof}

\section{Algorithm and Analysis: Proof of Theorem~\ref{thm:main}} \label{sec:alg}

Our algorithm is given in pseudocode below.

\bigskip

\fbox{%
  \begin{minipage}{0.95 \linewidth}
    \textbf{Algorithm} \textsc{Learn-One-Hidden-Layer-Networks}

\begin{enumerate}
\item Let $C>0$ be a sufficiently large universal constant.
\item \label{step2}For each $m=1,2,\ldots,4k$
use the algorithm from Lemma \ref{moment computation lemma}
to compute tensors $T_{m}$ such that with $99\%$ probability
$\|T_m - \E[F(X)H_m(X)]\|_2 < (\eps/k)^{Ck}$ for all such $m$.

\item Define a quadratic form on $\R^d$ by $Q(v) \eqdef \sum_{m=1}^{4k} \|T_m v\|_2^2$,
where $T_m v$ denotes the result of dotting the tensor $T_m$ with $v$ along one of its coordinates.

\item Let $V$ be the subspace spanned by the $k$ largest eigenvalues of $Q$.

\item \label{step5} For $m=0,1,2,\ldots, D$, where $D\eqdef C \eps^{-4/3}$, 
use the algorithm from Lemma \ref{moment computation lemma} to compute tensors $P_m$
such that $\|P_m - \E[F(X) H_m^{(V)}(\mathrm{Proj}_V(X))]\|_2 < \eps^2/(DC) = \eps^{10/3}/C^2$
with $99\%$ probability  for all such $m$.

\item Return the hypothesis function $\tilde F(x) \eqdef \sum_{m=0}^{D} P_m H_m(x).$
\end{enumerate}

\smallskip
  \end{minipage}%
}

\bigskip

Before proving correctness,
we analyze the sample complexity of Steps~\ref{step2} and~\ref{step5}.
We use Lemma \ref{moment computation lemma} with $t=m$. We note that
$$
\|F(X)\|_m \leq \sum_{i=1}^k |w_i| \|\relu(v_i\cdot X)\|_m \leq
\sum_{i=1}^k  |w_i| \|v_i\cdot X\|_m = O(\sqrt{m}) \sum_{i=1}^k |w_i|  = O(\sqrt{m}).
$$
In Step~\ref{step2}, the $\binom{d+m}{m}$ term is $d^{O(m)} = d^{O(k)}$,
$\delta = (\eps/k)^{O(k)}$ and $\tau = \Omega(1/k)$.
Thus, the sample complexity of this step is $(dk/\eps)^{O(k)}$.

In Step~\ref{step5}, since we may do this computation within $V$, 
which is a $k$-dimensional subspace, the $\binom{d+m}{m}$ term is $(1/\eps)^{O(k)}$,
giving a similar sample complexity bound.

Thus, the total sample complexity is $N = (dk/\eps)^{O(k)}$. It is also easy to see
that the runtime of the algorithm is sample polynomial.

We now proceed to prove correctness. First, 
we would like to analyze $V$ and in particular claim
that it is close in a sense to the span of the $v_i$'s. In particular, let
\begin{equation} \label{eqn:moment-tensor}
M_m \eqdef \E[F(X) H_m(X)] = c_m \sum_{i=1}^k w_i v_i^{\otimes m} \;,
\end{equation}
where the equation uses Corollary~\ref{cor:relu-corr},
and $c_m$ is defined in Lemma~\ref{lem:univ-relu-corr}.
Assuming that our algorithm in Step~\ref{step2} succeeds,
we have that $\|T_m - M_m\|_2 < (\eps/k)^{Ck}$ for all $m\leq 4k $.

We next define the quadratic form $Q_0(v)$ by
\begin{equation} \label{eqn:q0}
Q_0(v) \eqdef \sum_{m=1}^{4k} \| M_m v \|_2^2 =
\sum_{m=1}^{4k} c_m^2 \left\| \sum_{i=1}^k w_i (v\cdot v_i) v_i^{\otimes m-1} \right\|_2^2 \;,
\end{equation}
where the equation follows from  \eqref{eqn:moment-tensor}.
Since $\|T_m - M_m\|_2$ is small for all $m\leq 4k$,
for any unit vector $v$ it holds
$|Q(v)-Q_0(v)| < (\eps/k)^{Ck/2}$.
Furthermore, if $W$ is the space spanned by the $v_i$'s (which has dimension at most $k$),
then $Q_0$ vanishes on $W$. Therefore, if $v$ is any unit vector perpendicular to $V$
we have that:
\begin{align*}
|Q_0(v)| & \leq |Q(v)| + (\eps/k)^{Ck/2}\\
& \leq \sup_{w\in W^\perp, \|w\|_2=1} |Q(w)| + (\eps/k)^{Ck/2}\\
& \leq \sup_{w\in W^\perp, \|w\|_2=1}|Q_0(w)| + 2 (\eps/k)^{Ck/2}\\
& = 2 (\eps/k)^{Ck/2} \;,
\end{align*}
where the second line above follows from the variational formulation
of the principal value decomposition.
We conclude: 
\begin{lemma}\label{T on Vperp Bound Lemma}
For $v$ a unit vector perpendicular to $V$ and $m\leq 4k$, we have 
$\left\|M_m v \right\|_2^2 < 2 (\eps/k)^{Ck/2} \;.$
\end{lemma}

Next we would like to claim that $\|P_m-M_m\|_2$ is small for all $m$.
To this end, we establish the following proposition.

\begin{proposition} \label{prop:p-m}
For $m\leq D$, we have
$\|(P_m-M_m) \|_2 < 2\eps^2/(DC)$.
\end{proposition}
\begin{proof}
Note that by construction (Step~\ref{step5} of pseudocode)
$P_m$ is close to the projection of $M_m$ onto $V$.
In particular, if we let
$$
R_m \eqdef \mathrm{Proj}_V^{\otimes m}M_m = c_m \sum_{i=1}^k w_i \mathrm{Proj}_V(v_i)^{\otimes m},
$$
then $\|P_m - R_m\|_2 < \eps^{2}/(DC)$. Since,
$$
\|(P_m-M_m)\|_2 \leq \|(P_m-R_m)\|_2 + \|(R_m-M_m)\|_2,
$$
it remains to bound $\|R_m-M_m\|_2$.

Note that
\begin{equation} \label{eqn:rm-mm}
(R_m - M_m)/c_m = \sum_{i=1}^k w_i (v_i^{\otimes m} - \mathrm{Proj}_V(v_{i})^{\otimes m}) \;.
\end{equation}
For $m>4k$ and odd, \eqref{eqn:rm-mm} is $0$ because $c_m = 0$.
For $m> 4k$ and even, applying Corollary \ref{main bound cor} 
along with the fact that $|c_m| = O(1)$, we conclude that
$$
\|R_m - M_m\|_2 \leq O\left(\binom{m}{2k-1}(4k)^{2k}\max_{t<4k,\textrm{ even}}\|R_t -M_t\|_2 /c_t \;. \right)
$$
For $m\leq D$ the $\binom{m}{2k-1}(4k)^{2k}$ term is $(k/\eps)^{O(k)}$.
It remains to bound $\|R_t-M_t\|_2 /c_t$ when $t$ is even and at most $4k$.

Note that if $W$ is the span of the $v_i$'s, then
$R_t-M_t$ is in $(V+W)^{\otimes t}$. 
Let $x_1,\ldots,x_{2k}$ be an orthonormal basis of $V+W$
with $x_1,\ldots,x_k$ an orthonormal basis of $V$.
We will bound the $x_{i_1}x_{i_2}\ldots x_{i_t}$ entry of $R_t-M_t$.
In particular, if all of the $x_{i_j}$ are in $V$,
we have that since $R_t$ is the projection onto $V$ of $M_t$
that the corresponding coefficient is $0$. If, on the other hand, one of them (say $x_{i_1}$)
is perpendicular to $V$, then $R_t x_{i_1} = 0$ and 
$$
\|R_t x_{i_1}  - M_t x_{i_1}\|_2 = \|M_t x_{i_1}\|_2 \leq  2 (\eps/k)^{Ck/2} \;,
$$
where 
the inequality follows from Lemma \ref{T on Vperp Bound Lemma}.
Summing over all entries and using the fact that $c_t  = \Omega(t^{-5/4})$, 
we find that
$$\|R_t-M_t\|_2  / c_t < O(D^{5/4}) \, D^{O(k)} \, (\eps/k)^{Ck/2}  
< (\eps/k)^{Ck/3} \;.$$
for $C$ a sufficiently large universal constant.
This completes the proof of Proposition~\ref{prop:p-m}.
\end{proof}

We are now ready to bound the final error
and complete the proof of Theorem~\ref{thm:main}.
Note that
$$
\tilde F(x) = \sum_{m=0}^D P_m(x) H_m(x) \quad \textrm{ and }\quad F(x) = \sum_{m=0}^\infty M_m(x) H_m(x) \;.
$$
We can write
\begin{align*}
\|\tilde F(X) - F(X)\|_2^2
& = \sum_{m=0}^D \|P_m-M_m\|_2^2 + \sum_{m=D+1}^\infty \|M_m\|_2^2\\
& \leq \sum_{m=0}^D 2\eps^2/(CD) + \sum_{m=D+1}^\infty c_m^2 \left\| \sum_{i=1}^k w_i v_i^{\otimes m} \right\|_2^2\\
& \leq 2 \eps^2/C +\sum_{m=D+1}^\infty c_m^2 \left( \sum_{i=1}^k |w_i| \|v_i^{\otimes m}\|_2\right)^2\\
& \leq 2 \eps^2/C +\sum_{m=D+1}^\infty c_m^2\\
& \leq 2 \eps^2/C +\sum_{m=D+1}^\infty O(m^{-5/2})\\
& \leq 2 \eps^2/C + O(D^{-3/2})  < \eps^2 \;,
\end{align*}
where the first line follows by the orthonormality of the Hermite tensors,
the second line uses Proposition~\ref{prop:p-m}, and the fifth line uses the upper bound on $c_m$
from Lemma~\ref{lem:univ-relu-corr}.
This completes the proof of Theorem~\ref{thm:main}. \qed

\section{Conclusions} \label{sec:concl}

In this paper, we gave a simple algorithm that learns
one-hidden-layer ReLU networks of size $k$ 
under the Gaussian distribution on $\R^d$ to $L_2$-error $\eps$
with complexity $(dk/\eps)^{O(k)}$. While the complexity of our 
algorithm cannot be qualitatively improved within the class of CSQ
algorithms (a natural yet restricted family of algorithms), 
to the best of our knowledge there is no (known) 
inherent obstacle ruling out a $\poly(d, k, 1/\eps)$ time algorithm. 
It should be noted that the complexity of the non-CSQ  algorithm of~\cite{ChenKM21} 
is polynomial in $d$ but exponential in $1/\eps$ 
(even for constant $k$). The existence
of a fully-polynomial time algorithm remains open even for the special case
of positive weights, where the best known algorithm~\cite{DK20-ag} 
has runtime $\poly(d/\eps)+ (k/\eps)^{O(\log^2(k))}$.

\bibliographystyle{alpha}
\bibliography{allrefs}

\newcommand{\etalchar}[1]{$^{#1}$}
\begin{thebibliography}{GKLW19}

\bibitem[BJW19]{BakshiJW19}
A.~Bakshi, R.~Jayaram, and D.~P. Woodruff.
\newblock Learning two layer rectified neural networks in polynomial time.
\newblock In {\em Conference on Learning Theory, {COLT} 2019}, pages 195--268,
  2019.

\bibitem[Bon70]{Bon70}
A.~Bonami.
\newblock Etude des coefficients fourier des fonctiones de $l^{p}(g)$.
\newblock {\em Ann. Inst. Fourier (Grenoble)}, 20(2):335--402, 1970.

\bibitem[CDG{\etalchar{+}}23]{CDGKM23}
S.~Chen, Z.~Dou, S.~Goel, A.~R. Klivans, and R.~Meka.
\newblock Learning narrow one-hidden-layer relu networks.
\newblock {\em CoRR}, abs/2304.10524, 2023.
\newblock Conference version in COLT'23.

\bibitem[CGKM22]{ChenGKM22}
S.~Chen, A.~Gollakota, A.~R. Klivans, and Raghu Meka.
\newblock Hardness of noise-free learning for two-hidden-layer neural networks.
\newblock In {\em NeurIPS}, 2022.

\bibitem[CKM21]{ChenKM21}
S.~Chen, A.~R. Klivans, and R.~Meka.
\newblock Learning deep relu networks is fixed-parameter tractable.
\newblock In {\em 62nd {IEEE} Annual Symposium on Foundations of Computer
  Science, {FOCS} 2021}, pages 696--707. {IEEE}, 2021.

\bibitem[DFS16]{DanielyFS16}
A.~Daniely, R.~Frostig, and Y.~Singer.
\newblock Toward deeper understanding of neural networks: The power of
  initialization and a dual view on expressivity.
\newblock In {\em Advances in Neural Information Processing Systems 29: Annual
  Conference on Neural Information Processing Systems 2016}, pages 2253--2261,
  2016.

\bibitem[DGK{\etalchar{+}}20]{DGKK20}
I.~Diakonikolas, S.~Goel, S.~Karmalkar, A.~R. Klivans, and M.~Soltanolkotabi.
\newblock Approximation schemes for relu regression.
\newblock In {\em Conference on Learning Theory, {COLT} 2020}, volume 125 of
  {\em Proceedings of Machine Learning Research}, pages 1452--1485. {PMLR},
  2020.

\bibitem[DK20]{DK20-ag}
I.~Diakonikolas and D.~M. Kane.
\newblock Small covers for near-zero sets of polynomials and learning latent
  variable models.
\newblock In {\em 61st {IEEE} Annual Symposium on Foundations of Computer
  Science, {FOCS} 2020}, pages 184--195, 2020.
\newblock Full version available at https://arxiv.org/abs/2012.07774.

\bibitem[DKKZ20]{DKKZ20}
I.~Diakonikolas, D.~M. Kane, V.~Kontonis, and N.~Zarifis.
\newblock Algorithms and {SQ} lower bounds for {PAC} learning one-hidden-layer
  relu networks.
\newblock In {\em Conference on Learning Theory, {COLT} 2020}, volume 125 of
  {\em Proceedings of Machine Learning Research}, pages 1514--1539. {PMLR},
  2020.

\bibitem[DKMR22]{DKMR22-neuron}
I.~Diakonikolas, D.~Kane, P.~Manurangsi, and L.~Ren.
\newblock Hardness of learning a single neuron with adversarial label noise.
\newblock In {\em International Conference on Artificial Intelligence and
  Statistics, {AISTATS} 2022}, volume 151 of {\em Proceedings of Machine
  Learning Research}, pages 8199--8213. {PMLR}, 2022.

\bibitem[DKPZ21]{DiakonikolasKPZ21}
I.~Diakonikolas, D.~M. Kane, T.~Pittas, and N.~Zarifis.
\newblock The optimality of polynomial regression for agnostic learning under
  gaussian marginals in the {SQ} model.
\newblock In {\em Conference on Learning Theory, {COLT} 2021}, volume 134 of
  {\em Proceedings of Machine Learning Research}, pages 1552--1584. {PMLR},
  2021.

\bibitem[DKR23]{DKR23}
I.~Diakonikolas, D.~M. Kane, and L.~Ren.
\newblock Near-optimal cryptographic hardness of agnostically learning
  halfspaces and relu regression under gaussian marginals.
\newblock {\em CoRR}, abs/2302.06512, 2023.
\newblock Conference version in ICML'23.

\bibitem[DKRS22]{DKRS22}
I.~Diakonikolas, D.~Kane, L.~Ren, and Y.~Sun.
\newblock {SQ} lower bounds for learning single neurons with massart noise.
\newblock In {\em NeurIPS}, 2022.

\bibitem[DKTZ22]{DKTZ22}
I.~Diakonikolas, V.~Kontonis, C.~Tzamos, and N.~Zarifis.
\newblock Learning a single neuron with adversarial label noise via gradient
  descent.
\newblock In {\em Conference on Learning Theory}, volume 178 of {\em
  Proceedings of Machine Learning Research}, pages 4313--4361. {PMLR}, 2022.

\bibitem[DKZ20]{DKZ20}
I.~Diakonikolas, D.~M. Kane, and N.~Zarifis.
\newblock Near-optimal {SQ} lower bounds for agnostically learning halfspaces
  and relus under gaussian marginals.
\newblock In {\em Advances in Neural Information Processing Systems 33: Annual
  Conference on Neural Information Processing Systems 2020, NeurIPS 2020},
  2020.

\bibitem[DPT21]{DPT21}
I.~Diakonikolas, J.~H. Park, and C.~Tzamos.
\newblock Relu regression with massart noise.
\newblock In {\em Advances in Neural Information Processing Systems 34: Annual
  Conference on Neural Information Processing Systems 2021, NeurIPS 2021},
  pages 25891--25903, 2021.

\bibitem[GGJ{\etalchar{+}}20]{GoelGJKK20}
S.~Goel, A.~Gollakota, Z.~Jin, S.~Karmalkar, and A.~R. Klivans.
\newblock Superpolynomial lower bounds for learning one-layer neural networks
  using gradient descent.
\newblock In {\em Proceedings of the 37th International Conference on Machine
  Learning, {ICML} 2020}, volume 119 of {\em Proceedings of Machine Learning
  Research}, pages 3587--3596, 2020.

\bibitem[GK19]{GoelK19}
S.~Goel and A.~R. Klivans.
\newblock Learning neural networks with two nonlinear layers in polynomial
  time.
\newblock In {\em Conference on Learning Theory, {COLT} 2019}, pages
  1470--1499, 2019.

\bibitem[GKKT17]{GoelKKT17}
S.~Goel, V.~Kanade, A.~R. Klivans, and J.~Thaler.
\newblock Reliably learning the relu in polynomial time.
\newblock In {\em Proceedings of the 30th Conference on Learning Theory, {COLT}
  2017}, pages 1004--1042, 2017.

\bibitem[GKLW19]{GeKLW19}
R.~Ge, R.~Kuditipudi, Z.~Li, and X.~Wang.
\newblock Learning two-layer neural networks with symmetric inputs.
\newblock In {\em 7th International Conference on Learning Representations,
  {ICLR} 2019}, 2019.

\bibitem[GLM18]{GeLM18}
R.~Ge, J.~D. Lee, and T.~Ma.
\newblock Learning one-hidden-layer neural networks with landscape design.
\newblock In {\em 6th International Conference on Learning Representations,
  {ICLR} 2018}, 2018.

\bibitem[Gro75]{Gross:75}
L.~Gross.
\newblock Logarithmic {Sobolev} inequalities.
\newblock {\em Amer.\ J.\ Math.}, 97(4):1061--1083, 1975.

\bibitem[JSA15]{Janz15}
M.~Janzamin, H.~Sedghi, and A.~Anandkumar.
\newblock Beating the perils of non-convexity: Guaranteed training of neural
  networks using tensor methods, 2015.

\bibitem[Kea98]{Kearns:98}
M.~J. Kearns.
\newblock Efficient noise-tolerant learning from statistical queries.
\newblock {\em Journal of the ACM}, 45(6):983--1006, 1998.

\bibitem[SJA16]{SedghiJA16}
H.~Sedghi, M.~Janzamin, and A.~Anandkumar.
\newblock Provable tensor methods for learning mixtures of generalized linear
  models.
\newblock In {\em Proceedings of the 19th International Conference on
  Artificial Intelligence and Statistics, {AISTATS} 2016}, pages 1223--1231,
  2016.

\bibitem[VW19]{VempalaW19}
S.~Vempala and J.~Wilmes.
\newblock Gradient descent for one-hidden-layer neural networks: Polynomial
  convergence and {SQ} lower bounds.
\newblock In {\em Conference on Learning Theory, {COLT} 2019}, pages
  3115--3117, 2019.

\bibitem[WZDD23]{WZDD23}
P.~Wang, N.~Zarifis, I.~Diakonikolas, and J.~Diakonikolas.
\newblock Robustly learning a single neuron via sharpness.
\newblock {\em CoRR}, abs/2306.07892, 2023.
\newblock Conference version in ICML'23.

\bibitem[ZLJ16]{ZhangLJ16}
Y.~Zhang, J.~D. Lee, and M.~I. Jordan.
\newblock L1-regularized neural networks are improperly learnable in polynomial
  time.
\newblock In {\em Proceedings of the 33nd International Conference on Machine
  Learning, {ICML} 2016}, pages 993--1001, 2016.

\bibitem[ZSJ{\etalchar{+}}17]{ZhongS0BD17}
K.~Zhong, Z.~Song, P.~Jain, P.~L. Bartlett, and I.~S. Dhillon.
\newblock Recovery guarantees for one-hidden-layer neural networks.
\newblock In {\em Proceedings of the 34th International Conference on Machine
  Learning, {ICML} 2017}, pages 4140--4149, 2017.

\end{thebibliography}

\newpage

\end{document}